\documentclass{itatnew}

\usepackage{tikz}
\usepackage{amsmath,amssymb}
\usepackage{bm}
\usepackage{booktabs}

\usepackage{xcolor}

\let\vec\bm
\DeclareMathOperator{\sign}{sign}
\DeclareMathOperator{\E}{\mathbb{E}}

\makeatletter
\def\blfootnote{\xdef\@thefnmark{}\@footnotetext}
\makeatother

\begin{document}

\title{Learning to segment from object sizes}

\author{Denis Baručić \and Jan Kybic}

\institute{Department of Cybernetics, \\
Faculty of Electrical Engineering, \\
Czech Technical University in Prague}

\maketitle
\blfootnote{Copyright \copyright 2022 for this paper by its authors. Use permitted under Creative Commons License Attribution 4.0 International (CC BY 4.0).}

\begin{abstract}
Deep learning has proved particularly useful for semantic segmentation, a~fundamental image analysis task. However, the standard deep learning methods need many training images with ground-truth pixel-wise annotations, which are usually laborious to obtain and, in some cases (e.g., medical images), require domain expertise. Therefore, instead of pixel-wise annotations, we focus on image annotations that are significantly easier to acquire but still informative, namely the size of foreground objects. We define the object size as the maximum Chebyshev distance between a~foreground and the nearest background pixel. We propose an algorithm for training a~deep segmentation network from a~dataset of a~few pixel-wise annotated images and many images with known object sizes. The algorithm minimizes a~discrete (non-differentiable) loss function defined over the object sizes by sampling the gradient and then using the standard back-propagation algorithm. Experiments show that the new approach improves the segmentation performance.
\end{abstract}


\section{Introduction}

Semantic segmentation is the process of associating a class label to each pixel of an image. With the advent of deep learning, deep networks have achieved incredible performance on many image processing tasks, including semantic segmentation. Deep learning for semantic segmentation has many benefits; for example, it is flexible w.r.t. the model architecture and scales particularly well  \cite{liu2021,minaee2021}. On the contrary, the standard deep learning demands many \textit{ground-truth} (GT) pixel-wise annotations to prevent overfitting. Since a~human expert annotator must usually provide the GT annotations, acquiring a~good-quality training dataset can be difficult. To combat this issue, we focus on learning from GT image annotations that are easier to produce but still informative enough, namely the sizes of foreground objects. In practice, our approach assumes a~training dataset that consists of relatively few pixel-wise annotated images and many images with known object sizes. We present a~work-in-progress solution.

\subsection{Proposed approach}

Suppose a~standard convolutional network for image segmentation (e.g., a U-Net~\cite{ronneberger2015}). Given an input image, we feed it to the network and collect the output prediction. The prediction is then thresholded to obtain a~binary mask, which is processed by a~distance transform, assigning to each foreground pixel the shortest distance to the background. Finally, the object size is defined as double the maximum of the computed distances.

Due to the thresholding, the cost function is not differentiable and it is therefore not possible to use the standard gradient descent for learning. We overcome this obstacle by adding random noise to the output of our network. The predicted binary masks then become stochastic and the gradient can be sampled. A~detailed description of our method is given later in Sec.~\ref{sec:model} and \ref{sec:learning}.

\subsection{Related work}

Cano-Espinosa et al.~\cite{cano2020} considered a~similar learning problem. They proposed a~network architecture that performs a~biomarker (fat contents) regression and image segmentation after being trained directly on images annotated by biomarker values only. Similarly to ours, their method derives the biomarker value from the predicted segmentation deterministically. The difference is that their biomarker, equivalent to the foreground area, can be obtained by a~simple summation. Furthermore, the method assumes that the foreground objects can be roughly segmented using thresholding. P{\'e}rez-Pelegr{\'\i} et al.~\cite{perez2021} took a~similar approach. Although their method does not involve thresholding to produce approximate segmentation, it was tailored explicitly for learning from images annotated by the foreground volume (as their images are 3D).

Karam et al.~\cite{karam2019} implemented a~differentiable distance transform via a~combination of the convolution operations. The method is fast but exhibits numerical instabilities for bigger images. Resolving the numerical instabilities, Pham et al.~\cite{pham2020} later proposed a~cascaded procedure with locally restricted convolutional distance transforms. Nonetheless, both methods substitute the minimum function with the \textit{log-sum-exp} operation, which leads to inaccurate results.

The way our method deals with a~non-differentiable cost function is borrowed from stochastic binary networks~\cite{raiko2015}. In a~stochastic binary network, one needs to deal with zero gradient after each layer of the network. However, methods such as ARM~\cite{yin2019} or PSA~\cite{shekhovtsov2020} are unnecessarily complex. Instead, we employ a~single sample estimation, which has been discussed in \cite{yulai2019}.

\section{Model}
\label{sec:model}

The proposed model consists of (1) a~segmentation network, $f_{\vec{\theta}}$, parametrized by $\vec{\theta}$, and (2) a~deterministic algorithm to derive the object size based on distance transform, denoted as $g$.

Given an input image $\vec{x} = (x_1, \ldots, x_V)$, the network produces a~pixel-wise segmentation
\begin{equation}
\vec{a} = f_{\vec{\theta}}(\vec{x}),
\end{equation}
such that $a_i \in {\mathbb R}, \, 1 \leq i \leq V$, where $V$ is the number of pixels. The method does not make any assumptions about the network's technical details, except that it can be trained using the standard back-propagation algorithm and gradient descent. In our experiments, we always employed a~U-Net \cite{ronneberger2015} with a~residual network encoder \cite{he2016} and a~mirroring decoder.

To obtain a~binary mask $\vec{\hat{y}} \in \{ \pm 1 \}^V$, the network response $\vec{a}$ is thresholded,
\begin{equation}
\hat{y}_i = \sign a_i.
\end{equation}

\subsection{Object size}

We use a~distance transform of the binary mask to define the object size (see Fig.~\ref{fig:object-size}). Distance transform assigns to each pixel the shortest distance to the background, i.e.,
\begin{equation}
d_i = \min_{j, \hat{y}_j = -1} \delta( i, j ), \quad i = 1, \ldots, V,
\end{equation}
where $\delta(i, j)$ is the Chebyshev $\ell_{\infty}$ distance. After that, we take double the maximum distance to define the object size,
\begin{equation}
\hat{s} = 2 \, \max_i \, d_i.
\end{equation}
The composition of the distance transform and the maximum aggregation is the object size, denoted as $g \colon \{\pm 1 \}^V \to \mathbb{R}$,
\begin{equation}
g(\hat{\vec{y}}) = 2 \, \max_i \min_{j, \hat{y}_j = -1} \delta( i, j ).
\end{equation}

% fig:object-size
\begin{figure}[t]
\centering
\begin{tikzpicture}[scale=1.5]
% Object outline
\draw[ultra thick, fill=gray!40] plot [smooth cycle, tension=.8] coordinates {(0,0) (1,.2) (2,1) (2,1.5) (1,1.4) (0,1.44) (-.9, 1.1) (-.9, .3)};

% Center point
\node [draw, circle, fill, inner sep=0pt, minimum size=3pt, label=west:$i$] (center) at (0.045, 0.72) {};

\draw[thick] (0,0) -- (center) node[midway, left] {$d_i$};
\draw[thick,densely dotted] (center) -- (0.09, 1.44);
\draw[<->] (0.144,-0.009) --  (0.09 + 0.144, 1.44 - 0.009) node[midway, right] {$\tilde{s}$};
\end{tikzpicture}
\caption{Illustrative example of an object and its derived size. The object is outlined by the thick boundary line. The point $i$ denotes the foreground pixel whose shortest distance to the background, $d_i$, is the highest among the pixels. The derived object size $\hat{s}=2 d_i$.}
\label{fig:object-size}
\end{figure}
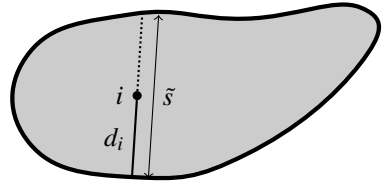

\subsubsection{Implementation details}

There is an efficient, \textit{two-pass} algorithm that computes the distance transform in $\Theta(V)$ time. Furthermore, when evaluating a~batch of images, it is possible to compute the distance transform on all images in parallel.

We have implemented a~CPU version\footnote{\url{https://github.com/barucden/chdt}} of this algorithm that works with PyTorch tensors and is faster than, e.g., the SciPy implementation.

\begin{figure*}[t]
\centering
\begin{tikzpicture}[
    minus/.style={path picture={
        \draw ([xshift=1.5] path picture bounding box.west) -- ([xshift=-1.5] path picture bounding box.east);
        }
    },
    stepfun/.style={inner sep=5pt, path picture={
        \draw (-0.15,-0.1) -- ( 0.0, -0.1) 
              ( 0.0, -0.1) -- ( 0.0,  0.1)
              ( 0.0,  0.1) -- (0.15,  0.1);
        }
    }
]

\node[circle, draw, fill, label=below:image $\vec{x}$, inner sep=0pt, minimum size=3pt] (input) at (0, 0) {};
\node[rectangle, draw, align=center] (net) at (2, 0) {Segmentation\\network $f_{\vec{\theta}}$};
\node[circle, minus, draw] (add) at (4, 0) {};
\node (noise) at (4, -1) {noise $Z$};
\node[rectangle, stepfun, draw] (threshold) at (5, 0) {};
\node[rectangle, draw, align=center] (distance) at (7, 0) {Distance\\transform};
\node[rectangle, draw] (maximum) at (9, 0) {Max};
\node[rectangle, draw] (loss) at (11.5, 0) {Loss $l$};
\node (label) at (11.5, -1) {$s$};
\node (output) at (13.5, 0) {$l(s, g(Y))$};

\draw[dotted] plot [smooth cycle, tension=0] coordinates {(6, -1.25) (9.65, -1.25) (9.65, .75) (6, .75)};
\node [anchor=south] at (7.825, -1.25) {Size derivation $g$};

\draw[->] (input) -- (net);
\draw[->] (net) -- (add);
\draw[->] (noise) -- (add);
\draw[->] (add) -- (threshold);
\draw[->] (threshold) -- (distance) node[midway, above] {$Y$};
\draw[->] (distance) -- (maximum);
\draw[->] (maximum) -- (loss) node[midway, above] {$g(Y)$};
\draw[->] (label) -- (loss);
\draw[->] (loss) -- (output);

\end{tikzpicture}%
\caption{An overview of the proposed probabilistic model.}
\label{fig:diagram}
\end{figure*}
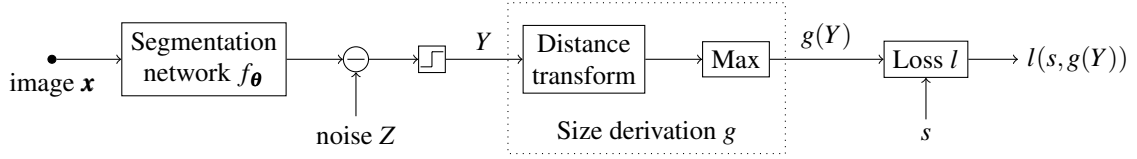

\section{Learning}
\label{sec:learning}

Suppose a~training dataset $\mathcal{D} = \mathcal{D}_f \cup \mathcal{D}_w$ consists of fully- and weakly-annotated subsets $\mathcal{D}_f$ and $\mathcal{D}_w$. The fully-annotated subset $\mathcal{D}_f$ contains pairs $(\vec{x}, \vec{y})$, where $\vec{x}$ is an input image and $\vec{y}$ the corresponding GT pixel-wise segmentation, while $\mathcal{D}_w$ comprises of pairs $(\vec{x}, s)$, where $s$ is the size of the object present in the image $\vec{x}$. We focus on situations when $\lvert \mathcal{D}_f \rvert \ll \lvert \mathcal{D}_w \rvert$.

\subsection{Supervised pre-training}

Our method starts by optimizing a~pixel-wise loss w.r.t. the network parameters $\vec{\theta}$ on the small subset $\mathcal{D}_f$, as in the standard supervised learning. For a~particular training pair $(\vec{x}, \vec{y}) \in \mathcal{D}_f$ and the corresponding prediction $\vec{a} \in \mathbb{R}^V$, the loss function reads
\begin{equation}
\sum_{i=1}^V \left( a_i(1 - y_i) + \log(1 + \exp(-a_i)) \right),
\end{equation}
which is sometimes referred to as the binary cross-entropy with logits loss. The optimization continues until convergence.

Using proper data augmentation to extend the training dataset, the network tends to recognize useful features and produces decent predictions after this initial stage (see Sec.~\ref{sec:pretraining-impact}).

\subsection{Weakly-supervised training}

Consider a~training pair $(\vec{x}, s) \in \mathcal{D}_w$. As described  in Sec.~\ref{sec:model}, one can obtain a~prediction of the object size, $\hat{s} = g(\hat{\vec{y}})$, from the thresholded network response $\hat{\vec{y}}$. We penalize the prediction error by the square loss
\begin{equation}
\label{eq:loss}
l(s, \hat{s}) = (s - \hat{s})^2.
\end{equation}

We propose to follow an approach similar to those used in binary neural networks \cite{shekhovtsov2020} and subtract random noise $Z$ from the real predictions $a_i$ before thresholding. Consequently, the binary segmentation becomes a~collection $\vec{Y} = (Y_1, \ldots, Y_V)$ of $V$ independent Bernoulli variables,
\begin{equation}
Y_i = \sign(a_i - Z),
\end{equation}
with
\begin{equation}
\label{eq:bernoulli-prob}
\Pr(Y_i = +1 \mid \vec{x}; \vec{\theta}) = \Pr(Z \leq a_i) = F_Z(a_i),
\end{equation}
where $F_Z$ is the cumulative distribution function (CDF) of the noise $Z$ (see Fig.~\ref{fig:diagram}).

Then, instead of minimizing the loss $l$ \eqref{eq:loss}, we minimize the expected loss $\mathcal{L}=\E_{\vec{Y}}[l(s, g(\vec{Y}))]$,
\begin{equation}
\label{eq:expected-loss}
\mathcal{L} = \sum_{\vec{y} \in \{\pm 1\}^V} \Pr(\vec{Y} = \vec{y} \mid \vec{x}; \vec{\theta}) l(s, g(\vec{y})).
\end{equation}
Contrary to \eqref{eq:loss}, the expected loss \eqref{eq:expected-loss} is differentiable, assuming a~smooth $F_Z$.

\subsubsection{Noise distribution}

Following \cite{shekhovtsov2020}, we sample the noise $Z$ from the logistic distribution with mean $\mu = 0$ and scale $s=1$. Hence, the CDF of $Z$ is a~smooth, sigmoid function,
\begin{equation}
F_Z(a) = \frac{1}{1 + \exp(-a)}.
\end{equation}

\subsubsection{Exact gradient}

To compute the gradient $\nabla_{\vec{\theta}} \mathcal{L}$, we need to evaluate the derivative
\begin{equation}
\label{eq:derivative}
\frac{\partial \E_{\vec{Y}}[l(s, g(\vec{Y}))]}{\partial F_Z(a_i)}
\end{equation}
for each pixel $i=1,\ldots,V$. The gradient can be then computed automatically by the back-propagation algorithm. However, an exact computation of \eqref{eq:derivative} leads to
\begin{equation}
\sum_{\vec{y} \in \{\pm 1\}^V} \frac{\Pr(\vec{Y} = \vec{y} \mid \vec{x}; \vec{\theta})}{\Pr(Y_i = y_i \mid \vec{x}; \vec{\theta})} l(s, g(\vec{y})) y_i,
\end{equation}
which involves summing $2^V$ terms and is thus tractable only for very small images. Instead, we resort to a~single sample estimator.

\subsubsection{Single sample estimator}

The single sample estimator is based on Lemma~\ref{lemma:estimator}, which is, in fact, a~specific form of \cite[Lemma B.1]{shekhovtsov2020}.

\begin{lemma}
\label{lemma:estimator}
Let $\vec{Y}=(Y_1, \ldots, Y_V)$ be a~collection of $V$ independent $\{ \pm 1 \}$-valued Bernoulli variables with probabilities $\Pr(Y_i = +1) = p_i$. Let $h$ be a~function $h \colon \{\pm 1\}^V \to \mathbb{R}$. Let $\vec{y} = (y_1, \ldots, y_V)$ denote a~random sample of $\vec{Y}$ and $\vec{y}_{\downarrow i} = (y_1, \ldots, y_{i-1}, -y_i, y_{i+1}, \ldots, y_V)$. Then
\begin{equation}
\label{lemma:eq:estimate}
y_i \left(h(\vec{y}) - h(\vec{y}_{\downarrow i}) \right)
\end{equation}
is an unbiased estimate of $\frac{\partial}{\partial p_i} \E_{\vec{y} \sim \vec{Y}}[h(\vec{y})]$.
\end{lemma}
\begin{proof}
We take the derivative of the expectation,
\begin{equation}
\label{lemma:eq:derivative}
\frac{\partial}{\partial p_i} \E_{\vec{y} \sim \vec{Y}}[h(\vec{y})] = \sum_{\vec{y}} \frac{\Pr(\vec{y})}{\Pr(y_i)} h(\vec{y}) y_i,
\end{equation}
and write out the sum over $y_i$,
\begin{equation}
\sum_{\vec{y}_{\neg i}}\sum_{y_i} \Pr(\vec{y}_{\neg i}) h(\vec{y}) y_i
=\sum_{\vec{y}_{\neg i}} \Pr(\vec{y}_{\neg i}) \sum_{y_i}  h(\vec{y}) y_i
\end{equation}
where $\vec{y}_{\neg i}$ denotes vector $\vec{y}$ with the $i$-th component omitted. Notice that the inner sum simplifies and no longer depends on $y_i$,
\begin{equation}
\sum_{\vec{y}_{\neg i}} \Pr(\vec{y}_{\neg i}) (h(\vec{y}_{i=+1}) - h(\vec{y}_{i=-1})),
\end{equation}
where $\vec{y}_{i=z}$ is the vector $\vec{y}$ with the $i$-th component set to $z$. Then, we multiply the inner subtraction by the constant factor $1=p_i + (1 - p_i) = \sum_{y_i} \Pr(y_i)$,
\begin{equation}
\sum_{\vec{y}_{\neg i}} \Pr(\vec{y}_{\neg i}) \sum_{y_i} \Pr(y_i) (h(\vec{y}_{i=+1}) - h(\vec{y}_{i=-1})),
\end{equation}
ultimately leading to the following expression for \eqref{lemma:eq:derivative}:
\begin{equation}
\sum_{\vec{y}} \Pr(\vec{y}) (h(\vec{y}_{i=+1}) - h(\vec{y}_{i=-1})),
\end{equation}
which can be written as
\begin{equation}
\sum_{\vec{y}}\Pr(\vec{y}) y_i \left[h(\vec{y}) - h(\vec{y}_{\downarrow i})\right].
\end{equation}
Thus, \eqref{lemma:eq:estimate} is a~single sample unbiased estimate of \eqref{lemma:eq:derivative}.
\end{proof}

\begin{figure}
\centering
\begin{minipage}[t]{.19\linewidth} \centering
\frame{\includegraphics[width=\linewidth]{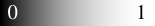}}

\vspace{1pt}

\frame{\includegraphics[width=\linewidth]{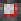}}

$F_Z$
\end{minipage}
\begin{minipage}[t]{.798\linewidth} \centering
\frame{\includegraphics[width=\linewidth]{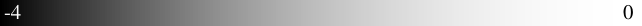}}

\vspace{1pt}

\begin{minipage}[t]{.24\linewidth} \centering
\frame{\includegraphics[width=\linewidth,page=1]{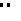}}

$n=1$
\end{minipage}
\begin{minipage}[t]{.24\linewidth} \centering
\frame{\includegraphics[width=\linewidth,page=2]{figs/grads.pdf}}

$n=8$
\end{minipage}
\begin{minipage}[t]{.24\linewidth} \centering
\frame{\includegraphics[width=\linewidth,page=3]{figs/grads.pdf}}

$n=64$
\end{minipage}
\begin{minipage}[t]{.24\linewidth} \centering
\frame{\includegraphics[width=\linewidth,page=4]{figs/grads.pdf}}

$n=512$
\end{minipage}
\end{minipage}

\caption{Examples of derivatives \eqref{eq:derivative} computed according to \eqref{eq:estimator} for different number of samples $n$, given the output of $F_Z$, for a~small, $6 \times 6$ image. The red frame outlines the object.}
\label{fig:gradient-examples}
\end{figure}

According to Lemma~\ref{lemma:estimator}, an unbiased estimate of the derivative \eqref{eq:derivative} is
\begin{equation}
\label{eq:estimator}
\frac{\partial \E_Y[l(s, g(Y))]}{\partial F_Z(a_i)} \approx y_i \left[l(s, g(\vec{y})) - l(s, g(\vec{y}_{\downarrow i}))\right],
\end{equation}
where $\vec{y}$ is a~random sample of Bernoulli variables with probabilities \eqref{eq:bernoulli-prob} (see a~few examples of sampled derivatives in Fig.~\ref{fig:gradient-examples}).

\section{Experiments}

The proposed method was implemented in the PyTorch Lightning framework\footnote{\url{https://github.com/Lightning-AI/lightning}} using a~ResNet implementation from the Segmentation Models PyTorch library\footnote{\url{https://github.com/qubvel/segmentation_models.pytorch}}. The presented experiments were perfomed on a~server equipped with Intel Xeon Silver 4214R (2.40GHz) and NVIDIA GeForce RTX 2080 Ti.

\begin{figure}[t]
\centering
\frame{\includegraphics[width=.7\linewidth]{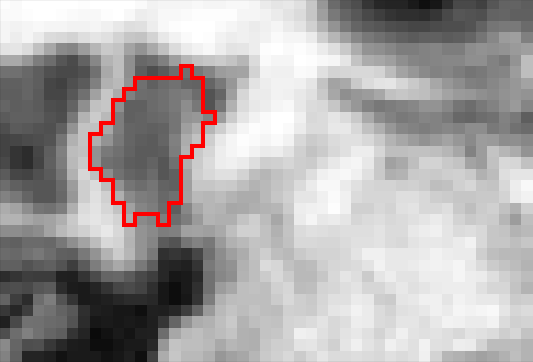}}
\caption{Example of a~hippocampus image~\cite{simpson2019} with the object outlined in red.}
\label{fig:hippocampus}
\end{figure}

The data for our experiments was based on a~dataset of 3D MRI images of the hippocampus \cite{simpson2019}. The dataset consists of 394 volumes provided with GT segmentation of classes \textit{hippocampus head}, \textit{hippocampus body}, and \textit{background}. We decomposed the volumes into individual 2D slices of size $48 \times 32$ pixels and kept only those with at least 1\% foreground, obtaining a~total of 6093 images. Next, we merged the \textit{hippocampus} classes to get a~binary segmentation problem (see Fig.~\ref{fig:hippocampus}). Afterward, we derived the object sizes from the GT pixel-wise annotations to use in training. Finally, we randomly split the data into training, validation, and testing subsets containing 70\%, 10\%, and 20\% of the images.

Given a~GT segmentation $\vec{y}$ and a~predicted segmentation $\hat{\vec{y}}$, we evaluate two metrics, the squared size prediction error $E$ and the intersection-over-union $IoU$,
\begin{align}
E(\vec{y}, \hat{\vec{y}}) &= l(g(\vec{y}), g(\hat{\vec{y}})), \\
IoU(\vec{y}, \hat{\vec{y}}) &= \frac{\sum_{i=1}^V 1 + y_i + \hat{y}_i + y_i \hat{y}_i}{\sum_{i=1}^V 3 + y_i + \hat{y}_i - y_i \hat{y}_i}.
\end{align}

In the case of standard supervised method, vertical and horizontal flipping was randomly applied to augment the training dataset. The proposed method did not apply any augmentation.

\subsection{Number of derivative samples}

\begin{figure}[t]
\includegraphics[width=\linewidth]{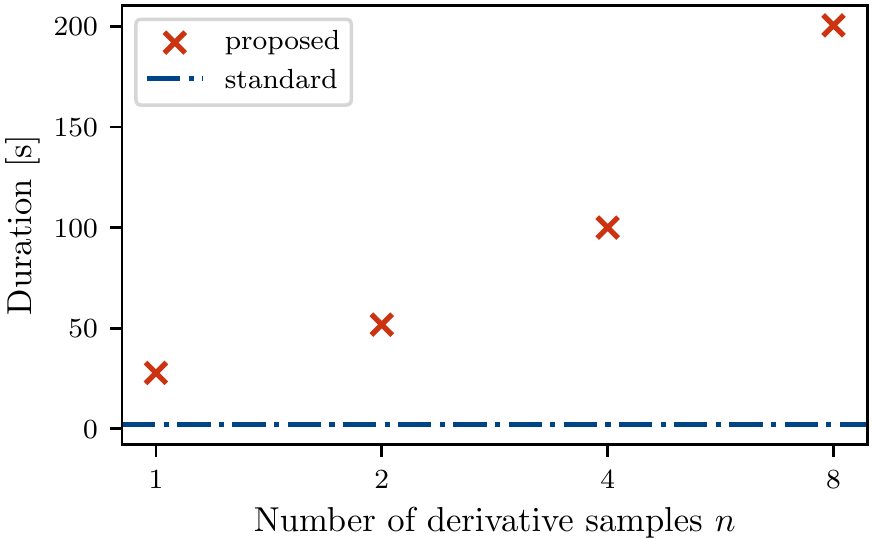}
\caption{Average epoch duration for the proposed method with different number of gradient samples. The duration of the standard method is given as a~reference.}
\label{fig:epoch-duration}
\end{figure}

\begin{figure*}
\centering
\includegraphics[width=.49\linewidth]{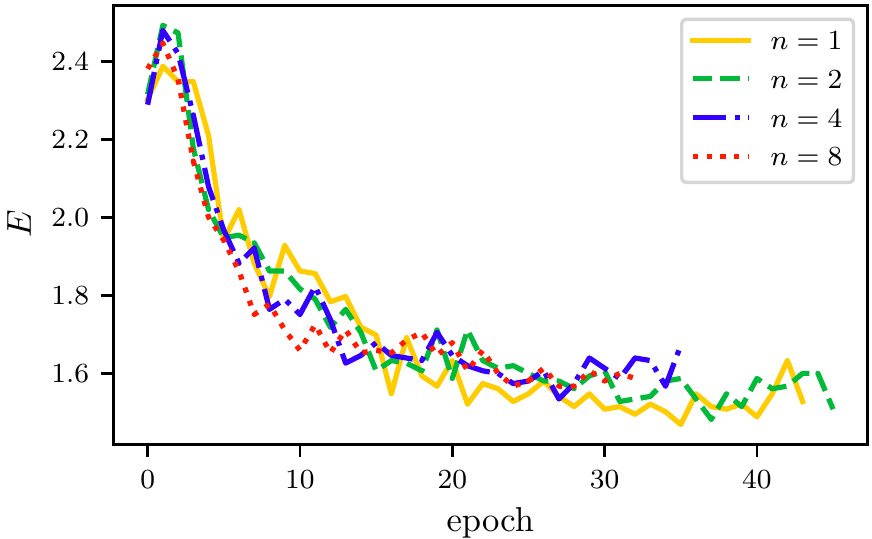}
\hfill
\includegraphics[width=.49\linewidth]{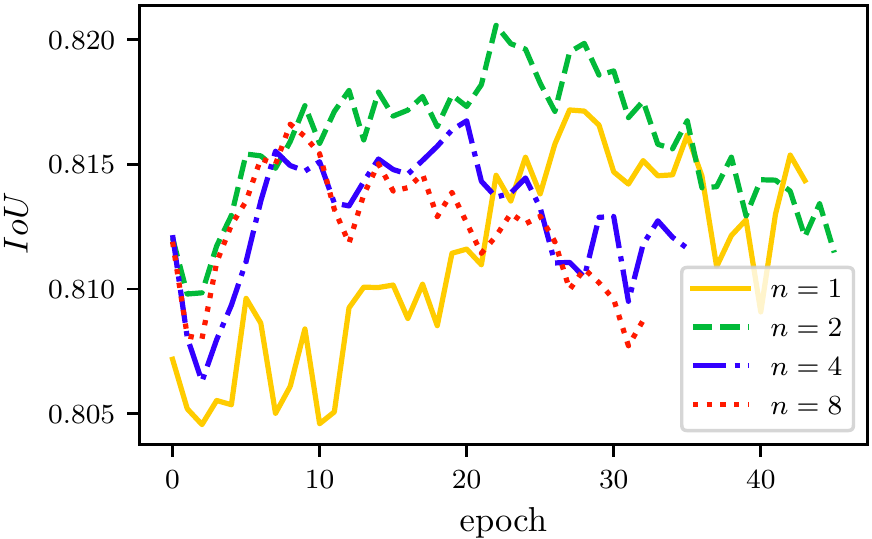}
\caption{Development of the squared size prediction error $E$ and the intersection-over-union $IoU$ on the validation images over the course of learning for different numbers of derivative samples $n$.}
\label{fig:error-and-iou}
\end{figure*}

A~toy example (see Fig.~\ref{fig:gradient-examples}) indicated that taking more samples of the derivatives \eqref{eq:estimator} might lead to better results than taking just one. This experiment investigates how the number of derivative samples $n$ impacts learning speed and prediction quality.

We considered four different numbers of samples $n$, $n \in \{1, 2, 4, 8\}$. For each $n$, the other parameters (such as the batch size or the learning rate) were the same, and the learning began with the same segmentation network $f_{\vec{\theta}}$ that was pre-trained in the standard way on $85$ pixel-wise annotated images from the training subset. The proposed method always ran until the squared error $E$ on the validation data stopped improving.

To assess the learning speed, we measured the duration of one learning epoch. For $n=1$, an epoch took $\approx 10\times$ longer than the standard supervised learning. Generally, the duration grew roughly exponentially with $n$ (see Fig.~\ref{fig:epoch-duration}).

Higher values of $n$ did not lead to a~lower $E$ or a~faster convergence speed (see Fig.~\ref{fig:error-and-iou}). In fact, $n=1$ and $n=2$ achieved the lowest $E$, but not by a~large margin. Given the speed benefits, we use $n=1$ always. Interestingly, even though $E$ kept decreasing over the course of learning for all $n$, $IoU$ improved only slightly and started declining after $\approx 20$ epochs. This observation suggests that the squared error of the object size is not a~sufficient objective for learning the segmentation.

\subsection{Pre-training impact}
\label{sec:pretraining-impact}

This experiment tests the essential question: given a~segmentation model trained on a~few pixel-level annotated images, can we improve its testing performance by further learning from size annotations?

We trained different segmentation networks until convergence on randomly selected training subsets of size $m$. Then, we fine-tuned these networks on the whole training dataset using the proposed method. We measured the test performance in terms of $IoU$.

The proposed method led to a~$\approx 5\%$ increase of $IoU$ for small $m < 100$ (see Fig.~\ref{fig:pretraining}), improving the segmentation quality.
For higher $m$, the effect was  negligible, which complements the observation from the previous experiment that improving the size estimate does not necessarily improve the segmentation quality.

\begin{figure}[t]
\includegraphics[width=\linewidth]{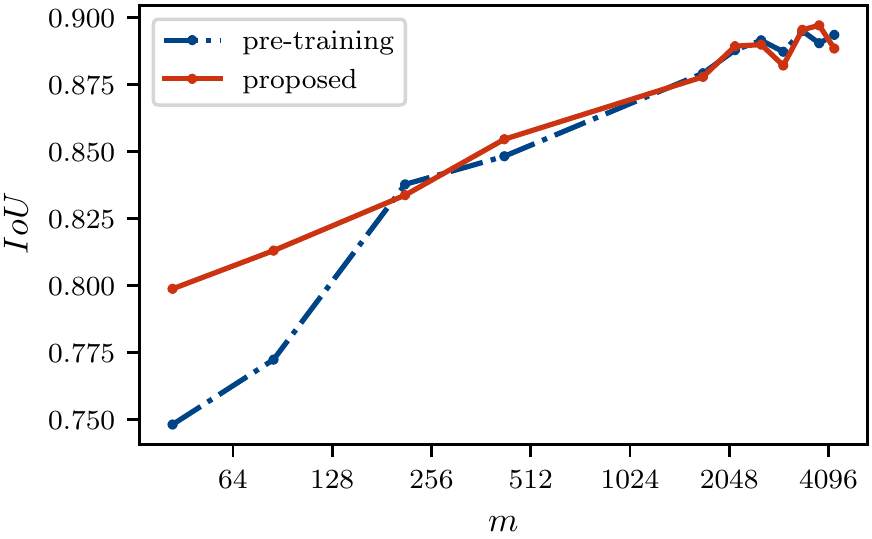}
\caption{$IoU$ on the test data for different sizes $m$ of the pre-training dataset. The plot shows results achieved by a~network after pre-training and after subsequent fine-tuning by the proposed method.}
\label{fig:pretraining}
\end{figure}

\section{Discussion}

The method is promising but there is definitely potential for improvement in both speed and prediction performance.

The proposed method samples the derivatives according to \eqref{eq:estimator} for each pixel $i$. Flipping the prediction, $y_i \mapsto -y_i$, changes the derived size only for some $i$; particularly those within and on the border of the predicted object. Therefore, given a~sample $\vec{y}$, $l(s, g(\vec{y})) = l(s, g(\vec{y}_{\downarrow i}))$ for many pixels $i$, and the sampled derivatives \eqref{eq:estimator} are sparse. The method might sample only those derivatives that are potentially non-zero and set the rest to zero directly, which would save much computational time.

We have seen in the experiments that lower size prediction error does not strictly imply better segmentation. We need to closely investigate in what cases the size prediction loss is insufficient and adjust the objective. The adjustment might involve adding an L1 regularization (as in \cite{cano2020}) or drawing inspiration from unsupervised methods (e.g., demand for the segmentation to respect edges in images, etc.).

The proposed approach entails some principled limitations. For example, it allows only a~single object in an image. We also expect the method to be ill-suited for complex object shapes, but we have not performed any experiments in that regard yet.

\section{Conclusion}

We proposed a~weakly-supervised method for training a~segmentation network from a~few pixel-wise annotated images and many images annotated by the object size.
The key ingredients is a~method for evaluating the object size from a probabilistic segmentation and a~method for optimizing a~deep network using a~non-differentiable objective.

The achieved results seem promising. We believe the improvements suggested in the discussion will improve performance, rendering the method valuable for training segmentation models for biomedical images.

\subsection*{Acknowledgments}

The authors acknowledge the support of the OP VVV funded project ``CZ.02.1.01/0.0/0.0/16\_019/0000765 Research Center for Informatics'', the Czech Science Foundation project 20-08452S, and the Grant Agency of the Czech Technical University in Prague, grant No. SGS20/170/OHK3/3T/13.

\bibliographystyle{hplain}
\bibliography{refs}

\end{document}